\documentclass[11pt,a4paper]{article}

\usepackage[a4paper]{geometry}
\usepackage[T1]{fontenc}
\usepackage{amsfonts,amsthm,amssymb,amsmath,mathdots, bbm,mathabx,mathrsfs, soul}

\usepackage{url}
\usepackage[latin1]{inputenc}
\usepackage{pdfpages}
\usepackage{epstopdf}
\usepackage{multirow}
\usepackage{footnote}
\usepackage{graphicx}
\usepackage{epsfig}
\usepackage{subfigure}
\usepackage{booktabs}
\usepackage{authblk}
\usepackage{cite}
\usepackage[colorlinks]{hyperref}
\hypersetup{
linkcolor=blue,
citecolor=blue,
}
\usepackage{dsfont}

\newtheorem{theorem}{Theorem}
\newtheorem{proposition}[theorem]{Proposition}%
\newcommand{\CT}[1]{\color{black} #1} 

\newcommand{\RL}[1]{\color{black} #1} 

\title{Low dimensional representation of multi-patient flow cytometry datasets using optimal transport for measurable residual disease detection in leukemia}

\author[1]{Erell Gachon}
\author[1]{J\'er\'emie Bigot}
\author[2]{Elsa Cazelles}
\author[3]{Audrey Bidet}
\author[3]{Jean-Philippe Vial}
\author[4]{Pierre-Yves Dumas}
\author[3]{Aguirre Mimoun}
\affil[1]{Institut de Math\'ematiques de Bordeaux, Universit\'e de Bordeaux, CNRS (UMR 5251)}
\affil[2]{CNRS, IRIT (UMR 5505), Universit\'e de Toulouse}
\affil[3]{CHU Bordeaux, Laboratoire d'H\'ematologie}
\affil[4]{CHU Bordeaux, Service d'H\'ematologie Clinique et
de Th\'erapie Cellulaire, Centre Hospitalier Universitaire de Bordeaux, F-33000 Bordeaux, France.}

\begin{document}

\maketitle

\abstract{
Representing and quantifying Measurable Residual Disease (MRD) in Acute Myeloid Leukemia (AML), a type of cancer that affects the blood and bone marrow, is essential in the prognosis and follow-up of AML patients. As traditional cytological analysis cannot detect leukemia cells below 5\%, the analysis of flow cytometry datasets is expected to provide more reliable results.
In this paper, we explore statistical learning methods based on  optimal transport (OT)  to achieve a relevant low-dimensional representation of multi-patient flow cytometry  measurements (FCM) datasets considered as  high-dimensional probability distributions. Using the framework of OT, we justify the use of the K-means algorithm for dimensionality reduction of multiple large-scale point clouds through mean measure quantization  by merging all the data into a single point cloud. After this quantization step, the visualization of the intra- and inter-patient FCM variability is carried out by  embedding  low-dimensional quantized probability measures into a linear space using either Wasserstein Principal Component Analysis (PCA) through linearized OT or log-ratio PCA of compositional data. Using a publicly available FCM dataset and a FCM dataset from Bordeaux University Hospital, we  demonstrate the benefits of our approach over the popular kernel mean embedding technique for statistical learning from multiple high-dimensional probability distributions. We also highlight the usefulness of our methodology for  low-dimensional projection and clustering patient measurements according to their level of MRD in AML from FCM. In particular, our OT-based approach allows a relevant and informative two-dimensional representation of the results of the FlowSom algorithm, a state-of-the-art method for the detection of MRD in AML using multi-patient FCM.\\
\textbf{Data Availability Statement:} All the analysis carried out in this work to analyze multiple FCM are reproducible through the use of Python codes that are available at \url{https://github.com/erellgachon/CytoLOT}.\\
\textbf{Correspondance:} Erell Gachon$^1$. Email :
 \href{erell.gachon@math.u-bordeaux.fr}{erell.gachon@math.u-bordeaux.fr}.\\
}

\section{Introduction}

Acute Myeloid Leukemia (AML) is a cancer of the blood and bone marrow when the development of blood stem cells (immature cells) into mature cells is altered, and the number of these immature cells, called myeloid blasts, increases. The diagnosis of AML consists in the detection of blasts which invade the bone marrow to great amounts, and it is typically assessed with the cytological analysis of the patient's bone marrow. The goal of the chemotherapy treatment is therefore to eradicate the blasts population: 80\% of patients reach complete remission after one cycle of chemotherapy, however 50\% relapse \cite{vial2021unsupervised}. The definition of complete remission states that the bone marrow should contain less than 5\% of blasts. However, detecting smaller numbers of persisting residual leukemic cells, called Measurable Residual Disease (MRD), is critical for disease monitoring and relapse prognosis. This motivates the need to catch the presence of blasts at very low thresholds (between 0.1\% and {\RL 0.004\%}). The detection and the quantification of MRD in AML has been widely studied as a predictive and prognostic tool \cite{dekker2023using, heuser20212021, short2022association, short2020association, cluzeau2022measurable, ivey2016assessment}. Molecular biology and flow cytometry are the main two techniques currently available to achieve informative sensitivity levels for anticipating the disease's progression. While molecular biology techniques for the assessment of MRD were the first to be developed and have a recognized decision-making value \cite{preudhomme2020maladie, badaoui2023suivi}, they have limited applicability (only available for 40\% of the patients) due to the lack of sufficiently specific molecular targets in 60\% of patients.

Flow cytometry  is a technology which allows the fast analysis of multiple parameters of a large population of cells from a biological sample. The interaction of the cells with the light emitted from different lasers enables the evaluation of two types of characteristics: the structure of the cell (size, granularity) and the type of proteins expressed on the surface of the analyzed cell.
For a given sample, each parameter is measured for each cell. One of the main goals of flow cytometry data analysis is then to identify the different types of cells in a biological sample, and to be able to locate regions of interest. Although flow cytometry datasets are widely used in practice, analysis of the datasets is not straightforward, as the large amount of data they generate and the dimensionality of the observations pose significant scientific challenges for their interpretation for MRD assessment purposes \cite{vial2021unsupervised}. Additionally surface markers used to characterize AMLs have highly variable expression levels from patient to patient, ranging from overexpression to expression negativity, and it is the combination of these abnormalities that defines the leukemic phenotype.

 The current preferred method for identifying cells is the so-called \emph{manual gating}. It is based on the successive projection of the point cloud of the cytometry measurements according to two selected parameters. The experts then partition the space according to the two selected parameters, visually identifying regions of high density. In the diagnosis of AML, the process of manual gating is well-suited and largely employed to the detection of blasts (tumor cells) which invade the body to great amounts (at thresholds of $10^{11}$ to $10^{12}$ {\RL cells in the body}), allowing easy identification of the tumor cell subpopulation. However, it presents a few shortcomings: it is highly dependent on the biologist's expertise and it is extremely time-consuming. Moreover, the significant variability in flow cytometry data is a limiting factor. This variability is induced both by the biological heterogeneity between the samples and by technical details such as the cytometer performance and the data acquisition parameters. Finally, because of the very low thresholds of MRD, the identification based on manual gating of cellular subpopulation of blasts in flow cytometry data during follow-up of patients treated with chemotherapy is made tricky.

For a given patient, a single dataset of flow cytometry  is usually made of measurements using 10-20 parameters from hundreds of thousands of cells. This motivates the need to develop dimensionality reduction tools for representing variations of interest within datasets acquired from an individual patient and across a population of patients using multiple flow cytometry measurements (FCM). To overcome the shortcomings of a manual approach to data gating in FCM, the development of automatic inference methods is currently a very active field of research in computational statistics and bioinformatics. Several methods have been introduced to automate the gating of cytometry data \cite{weber2016comparison}. One can cite
unsupervised algorithms such as FlowMeans \cite{aghaeepour2011rapid}, based on K-means, Cytometree \cite{commenges2018cytometree}, based on binary trees, or \cite{hejblum2019sequential} based on Dirichlet process mixtures. There also exist supervised algorithms such as DeepCyTOF \cite{li2017gating} or flowlearn \cite{lux2018flowlearn}. Regarding MRD detection for AML, \cite{vial2021unsupervised} proposes a method using FlowSOM \cite{van2015flowsom}, a visualization and clustering method based on self-organizing maps. In this paper, we will denote MRD-FlowS the MRD measure calculated by this method. On the other hand, MRD-BioM will designate the MRD value assessed with molecular biology. In \cite{nguyen2023computational}, the authors also use FlowSOM to detect MRD for chronic lymphocytic leukemia.

Optimal Transport (OT) provides a versatile and powerful framework for addressing a diverse array of problems involving point clouds modeled as discrete probability measures. In the last decade, the use of computational OT \cite{peyre2019computational} in data science and the related notion of Wasserstein distance have gained increasing interest in statistics (see e.g.\ \cite{bigotreview,panaretos2020invitation} for recent overviews), in particular for dimension reduction of a set of probability measures using the notion of Wasserstein Principal Component Analysis (PCA) \cite{bigot2017geodesic,cazelles:hal-01581699}. OT has also proved useful for understanding various phenomena in biology \cite{bellazzi2021gene, cao2022manifold, demetci2020gromov, schiebinger2019optimal}. For the purpose of automatic gating from flow cytometry data, OT has been previously leveraged for quantifying the distance between two point clouds, taking into account the geometry of the data in high-dimensional spaces, giving rise to two OT-based algorithms: optimalFlow \cite{del2020optimalflow}  proposes a method for supervised gating, and CytOpt \cite{freulon2023cytopt} estimates the different cell proportions of a population from cytometry measurements. More recently, \cite{mahan2024point} introduced a method based on linear optimal transport for the classification of point clouds and showed its applications to flow cytometry datasets from AML patients.

\subsection{Contributions} {\CT The novelty of our work compared to }{\RL previous studies \cite{basu2014detecting, wang2013linear} revolving around linear optimal transport and PCA} {\CT is} a statistical methodology based on computational OT for embedding into a low-dimensional linear space a set of flow cytometry  measurements (FCM) modeled  as  high-dimensional probability distributions. Our approach is suitable for the study of datasets that consist of multiple FCM from various patients  with possibly several measurements for a given patient (e.g.,  acquired at different time or from different cytometers). Using the notion of Wasserstein barycenter \cite{agueh2011barycenters}, we justify the use of the K-means algorithm for dimensionality reduction of multiple large-scale point clouds through mean measure quantization by merging all the data into a single point cloud.  This quantization step allows to reduce the size of the support of the observed data viewed as high-dimensional discrete probability measures.  The resulting low-dimensional support quantized probability measures are then embedded into a linear space using either Wasserstein  PCA \cite{bigot2017geodesic,cazelles:hal-01581699} through linearized  OT \cite{wang2013linear, basu2014detecting, moosmuller2023linear, mahan2024point} or log-ratio PCA of compositional data \cite{Aitchison83}. This enables the analysis of intra- and inter-patient FCM variability. Using a publicly available dataset and a dataset provided by Bordeaux University Hospital, this approach is first shown to achieve a meaningful  clustering of patients according to their FCM. We also  demonstrate the benefits of our methodology over the popular kernel mean embedding (KME) technique  \cite{muandet2017kernel}  for statistical learning from multiple high-dimensional probability distributions. 
 Secondly, we study applications to the assessment of MRD in AML.
In particular, our OT-based approach is found to provide a relevant and informative two-dimensional representation of the results obtained with the FlowSom algorithm (MRD-FlowS), a robust method with growing adoption in the field for the detection of MRD in AML using FCM. This representation could be a valuable tool to help biologists confirm or {\RL invalidate} the relapse prognosis of AML patients.

\subsection{Outline of the Paper}
In Sections \ref{sec:DataML} and \ref{sec:HIPC}, we first describe the two FCM datasets used to evaluate our methodology. We then introduce the FlowSOM method \cite{van2015flowsom} in Section \ref{sec:FlowSOM}. Section \ref{sec:OT} is devoted to defining OT, the Wasserstein barycenter, our proposed mean measure quantization-based dimension reduction method, and embedding in the OT sense. The visualization of datasets using PCA of two different embeddings is then presented in Section \ref{sec:PCA}. In Section \ref{sec:perf}, we define the silhouette score for performance evaluation purposes, while Section \ref{sec:KME} describes the kernel mean embedding technique. Finally, the results for the two FCM datasets are presented in Section \ref{sec:results}, and a discussion is given in Section \ref{sec:dis} to conclude the paper.

\section{Materials and Methods}
\label{sec:mat_methods}

\subsection{DATAML-Bordeaux Study}
\label{sec:DataML}

A first set of FCM for the statistical study conducted in this paper is provided by the Hematology Laboratory of Bordeaux University Hospital (H\^opital Haut-L\'ev\^eque). The dataset was collected from patients recruited between October 2019 and December 2022, diagnosed with acute myeloid leukemia (AML). They received intensive chemotherapy treatment and achieved complete remission after induction. This database is covered by the MR-004 of Bordeaux University Hospital. The Research Ethics Committee of the Bordeaux University Hospital certified this project and received a favorable opinion under the reference number CER-BDX 2024-164.

From this study, we have kept 182 patients and 22 normal bone marrow (NBM). For each patient, two or three FCM are available: one diagnosis ($\sim$50,000 cells) and one or two follow-ups (up to 500,000 cells). In total, we have 512 FCM. For each follow-up, we dispose of its associated MRD-FlowS estimation with the FlowSOM method. 

For the follow-ups, the MRD-BioM measure  is also given when available (108 out of 314 follow-ups), which we consider as ground truth. {\RL Concerning molecular biology, RNAs were purified from Bone Marrow mononuclear cells using Trizol reagent (Invitrogen-Thermo Fischer, Carlsbad, CA, USA). Reverse transcription was performed using the SuperScript IV VILO Master Mix (Lifetechnologies-Thermo Fischer, Carlsbad, CA, USA). The assessment of transcript levels was performed on a Light Cycler 480 (Roche LifeScience, Penzberg, Germany) with a specific RQ-PCR assay as previously described in \cite{kronke2011monitoring}. MRD levels were reported as the normalized values of fusion transcripts or NPM1mut copy number/ABL copy number $\times$ 100 (\%). The quantitative detection limit of the assays was 0.02\% for NPM1 mutations. The achievement of MRD levels below this threshold was defined as a negative MRD.}

There are 11 markers in the flow cytometry datasets : 2 are structural (FS INT, SS INT) and 9 are protein expressions (CD34, CD13, CD38, CD7, CD33, CD56, CD117, HLA-DR, CD45). Hence, the datasets consist of multiple point clouds in dimension $d=11$. We represent the $N = 512$ FCM as high-dimensional discrete probability measures to leverage OT statistical learning tools. For $1\leq i \leq N$, we represent the $i$-th dataset $X^i = (X^i_1,\cdots, X^i_{m_i}) \in(\mathbb{R}^d)^{m_i}$ with $m_i$ the number of cells in this dataset, as a discrete measure $\mu^i = \frac{1}{m_i}\sum_{j=1}^{m_i} \delta_{X_j^i}$.

 The cytometry files are stored in both FCS and CSV format and have already undergone standard pre-processing, i.e., transformation, compensation and normalization.

\subsection{HIPC dataset}
\label{sec:HIPC}

 Using a second dataset, that is publicly available on ImmuneSpace \cite{Data_HIPC}, we illustrate the  benefits of our methodology to cluster FCM data by analyzing observations from the T-cell panel of the Human Immunology Project Consortium (HIPC). Seven laboratories stained three replicates (denoted A, B, and C) of three cryo-preserved biological samples denoted patient 1, 2, and 3 (e.g., cytometry measurements from the Stanford laboratory for replicate C from patient 1). Each resulting dataset, that consists of the measurements of 7 markers for approximately $100000$ cells, thus leads to a set of FCM consisting of $N = 7 \times 3 \times 3 = 63$ point clouds in dimension $d=7$. {\RL The markers are the following : CCR7, CD4, CD45RA, CD3, HLADR, CD38, CD8}. For this dataset, a relevant clustering consists of three groups, each comprising data collected from a single patient.

\subsection{The FlowSOM Method}
\label{sec:FlowSOM}

FlowSOM \cite{van2015flowsom} is a visualization and clustering method that is specifically designed for analyzing flow cytometry data with Self-Organizing Maps (SOMs). This algorithm has been leveraged in \cite{vial2021unsupervised} for the evaluation of MRD in AML.

A SOM is an unsupervised technique for both clustering and dimensionality reduction: the training samples are represented through a low-dimensional discretized representation of the input space, called a map. Each node of the map is linked to a weight vector corresponding to a location in the input space. Close nodes in the map coincide with close points of the input space. In FlowSOM, the number of nodes chosen for the map is greater than the number of expected clusters, to guarantee the purity of clusters. The generated SOM is then visualized through a Minimum Spanning Tree (MST). The final step of FlowSOM is a meta-clustering process, to reduce the number of clusters. 

The method in \cite{vial2021unsupervised} to evaluate MRD in a follow-up of a patient undergoing treatment first consists of merging three cytometry entry files: a normal bone marrow (from an arbitrary, healthy donor), the diagnosis of the patient with AML, and his follow-up. The authors  in \cite{vial2021unsupervised} then perform a FlowSOM on the merged files, the last step of which results in an MST. They visualize the weights on the MST for each entry file. By comparing the three MSTs globally, they are able to identify the Nodes of Interest (NoIs) that are decisive in assessing the MRD present in the measurements. It is reported in \cite{vial2021unsupervised} that this method shows 80\% concordance with MRD-BioM, that is with the ground truth results. More precisely, MRD-FlowS yields 85\% specificity and 69\% sensitivity compared to MRD-BioM.

\subsection{Optimal Transport}
\label{sec:OT}

\subsubsection{Optimal Transport and Wasserstein Barycenter}

OT comes from a very concrete problem formulated by Gaspard Monge in 1781 \cite{Monge}: a worker with a shovel has to move a pile of sand of a certain shape into a hole of another shape. The aim is to minimize the global effort of the worker with regards to some cost function $c: \mathcal{X}\times \mathcal{Y}\rightarrow \mathbb{R}$, representing the \emph{cost} of moving mass from a point $x\in\mathcal{X}$ of the source, to a point $y\in\mathcal{Y}$ of the target. In this paper, we consider the squared Euclidean cost $c(x,y)=\|x-y\|^2$  with $\mathcal{X} = \mathcal{Y} = \mathbb{R}^d$. More formally, OT is  formulated as finding an optimal measurable map $T:\mathcal{X}\rightarrow \mathcal{Y}$ which sends a source discrete probability measure $\mu = \sum_{i=1}^n a_i\delta_{x_i}$ to a target measure $\nu = \sum_{j=1}^m b_j \delta_{y_j}$, with $a$ (resp. $b$) in the unit simplex $\Sigma_n$ (resp. $\Sigma_m$). The Monge problem then reads as
\begin{equation}\label{monge}
\min\limits_{T_{\#}\mu = \nu}\sum\limits_i a_i\|x_i-T(x_i)\|^2
\end{equation}
where $T_{\#}\mu$ denotes the push-forward measure of $\mu$ by a map $T: \{x_1,\cdots,x_n\}\rightarrow \{y_1,\cdots,y_m\}$ satisfying $\forall \ 1\leq j \leq m, b_j = \sum_{i:T(x_i)=y_j} a_i$. A solution of Problem \eqref{monge} is then called a Monge map.

The Kantorovich's relaxation of Monge's problem allows mass splitting of the source by introducing a so-called transport plan that is expressed as a coupling matrix $P\in \mathbb{R}_+^{n\times m}$ where $P_{ij}$ represents the number of units (or mass) moved from $x_i$ to $y_j$. Then, the Kantorovich formulation of optimal transport between $\mu$ and $\nu$ reads as
\begin{equation}\label{kanto}
\min\limits_{P\in\mathbb{R}_+^{n\times m}} \left\{    \sum\limits_{i,j} P_{ij}\|x_i-y_j\|^2 | \sum\limits_j P_{ij} = a_i , \sum\limits_i P_{ij} = b_j \right\}.
\end{equation}
If we denote $P^*$ an optimal plan between $\mu$ and $\nu$ in \eqref{kanto}, the 2-Wasserstein distance is given by
\begin{equation}\label{wassersteindistance}
    W_2(\mu,\nu) = \left(\sum\limits_{i,j} P^*_{ij}\|x_i-y_j\|^2 \right)^{1/2}.
\end{equation}
Using the above 2-Wasserstein distance, a notion of  averaging of measures is given by the Wasserstein barycenter  of a set of discrete measures $(\mu^i)_{1 \leq i \leq N}$ \cite{pmlr-v80-claici18a}, as the measure $\rho = \sum_{k=1}^K b_k \delta_{x_k}$ supported on $K$ points with weights $b$ in the probability simplex solving the following optimization problem: 

\begin{equation}\label{wbary}
 \min\limits_{b\in\Sigma_K, x\in(\mathbb{R}^d)^K} \frac{1}{N} \sum\limits_{i=1}^N W_2^2\left(\sum\limits_{k=1}^K b_k \delta_{x_k}, \mu^i\right).  
\end{equation}

{\RL However, using OT in problems involving point clouds suffers from high computational costs as it has complexity $O(n^3\log(n))$ \cite{peyre2019computational} for points clouds of $n$ points. Hence, when the number of points per clouds is larger than $10^5$, solving the exact OT problem becomes too expensive.}

\subsubsection{Dimension Reduction via Mean Measure Quantization} \label{sec:quantization}

A first contribution of our work is to use the OT framework to justify a classical pre-processing step in flow cytometry data analysis: summarizing spatial distribution of cells for each FCM sample into a small number $K$ of common $d$-dimensional points. Indeed, the dimensionality of the data (i.e. the number of cells per FCM sample) is prohibitive for the application of any statistical analysis based on computational OT on point clouds. For example, in the FlowSOM method \cite{vial2021unsupervised}, the authors read the information of the presence of blasts in the size of the clusters generated by the SOM for each dataset, focusing on clusters corresponding to cells compatible with blasts in a CD45/SS graph.

We then rely on Wasserstein's notion of average to address the synthesis of the data. However, the Wasserstein barycenter as defined in \eqref{wbary} cannot encode the information of the variation of the proportion of blasts between two types of samples, and therefore cannot generally differentiate the diagnosis and the follow-up. Consequently, we propose the following alternative Wasserstein barycenter problem with varying weights for each measure:

\begin{equation}\label{newbary}
 \min\limits_{a\in(\Sigma_K)^N, x\in(\mathbb{R}^d)^K} \frac{1}{N} \sum\limits_{i=1}^N W_2^2\left(\sum\limits_{k=1}^K a^i_k \delta_{x_k}, \mu^i\right).    
\end{equation}
The optimization problem \eqref{newbary} then provides $N$ weights vectors $(a^i)_{i=1}^N$ in the simplex $\Sigma_K$ associated to the $N$ input measures $(\mu^i)_{i=1}^N$, which results in approximating measures
\begin{equation}
\nu^i = \sum_{k=1}^K a^i_k \delta_{x_k}, \quad i=1,\ldots,N \label{eq:measuresnu}
\end{equation}
that share the same low-dimensional support $(x_1,\cdots,x_K)\in(\mathbb{R}^d)^K$, with $K$ typically small. The following proposition shows that solving Problem \eqref{newbary} amounts to solve the $K$-means problem for the mean measure $\overline{\mu} = \frac{1}{N} \sum_{i=1}^N \mu^i$ by merging all the data into a single point cloud. This dimension reduction step of multiple probability measure has been referred to as mean measure quantization in \cite{chazal21}, but its connection to the Wasserstein barycenter problem   \eqref{newbary} with varying weights has not been observed so far.

\begin{proposition}\label{prop1}
    When the $\mu^i$'s are absolutely continuous w.r.t. the Lebesgue measure, the Wasserstein barycenter problem with varying weights \eqref{newbary} is equivalent to the $K$-means quantization of the mean measure $\overline{\mu} = \frac{1}{N} \sum_{i=1}^N \mu^i$, that is,
    $$\min\limits_{a\in(\Sigma_K)^N, x\in(\mathbb{R}^d)^K} \frac{1}{N} \sum\limits_{i=1}^N  W_2^2\left(\sum\limits_{k=1}^K a^i_k \delta_{x_k}, \mu^i\right) =  \min\limits_{x\in(\mathbb{R}^d)^K} W_2^2\left(\sum\limits_{k=1}^K \overline{\mu}(V_{x_k}) \delta_{x_k}, \overline{\mu}\right)$$
where $\overline{\mu} = \frac{1}{N} \sum_{i=1}^N \mu^i$ is the average measure, and $V_{x_k}$ is the Voronoi cell centered at point $x_k$:
\begin{equation}\label{voronoi}
    V_{x_k} = \{z\in \mathbb{R}^d ~:~ \|z-x_k\|^2 \leq \|z-x_{k'}\|^2, ~ \forall k'\}.
\end{equation}
Moreover, a  solution $(a,x)  \in (\Sigma_K)^N \times (\mathbb{R}^d)^K$ of Problem \eqref{newbary} verifies
$$
a^i_k = \int_{V_{x_k}}\mathrm{d}\mu^i(y)=\mu^i(V_{x_k}).
$$

\end{proposition}
\begin{proof}[Proof of Proposition~{\upshape\ref{prop1}}.]

We first rewrite the functional to minimize in   \eqref{newbary} using the dual Kantorovich formulation of OT (see e.g. \cite{villani2009optimal})
\begin{align}
F(a,x):&= \frac{1}{N} \sum_{i=1}^N W_2^2\bigl(\sum_{k=1}^K a^i_k \delta_{x_k}, \mu^i\bigr)\\
 &= \frac{1}{N} \sum\limits_{i=1}^N \sup\limits_{\phi^i \in \mathbb{R}^K} \sum\limits_{k=1}^K \Bigl( a^i_k \phi^i_k + \int_{\mathbb{R}^d} \tilde{\phi}^i(y) \mathrm{d} \mu^i(y) \Bigr)
\end{align}
where $\phi^i\in\mathbb{R}^K$ is a Kantorovich potential and $\tilde{\phi}^i(y) = \inf_{1 \leq k \leq K} \{ \|x_k-y\|^2 - \phi^i_k \}$ is the $c$-transform of $\phi^i$. We denote $H(\phi^i, x, a^i)= \sum_{k=1}^K \Bigl( a^i_k \phi^i_k + \int_{\mathbb{R}^d} \tilde{\phi}^i(y) \mathrm{d} \mu^i(y) \Bigr)$, and the Laguerre cell $L_{x_k}(\phi^i) = \{z\in \mathbb{R}^d ~:~ \|z-x_k\|^2 - \phi^i_k \leq \|z-x_{k'}\|^2-\phi^i_{k'}, ~ \forall k'\}$. Note that $V_{x_k}= L_{x_k}(0)$.

First, we have that $\frac{\partial H}{\partial a^i_k} = \phi^i_k$. Additionally, the derivative of $H$ with respect to $\phi^i_k$ is given by
$$\frac{\partial H}{\partial \phi^i_k} = a^i_k - \int_{L_{x_k}(\phi^i)}\mathrm{d}\mu^i(y).$$
Then, by first order condition, we have $\phi^i_k=0$ and $a^i_k = \int_{V_{x_k}}\mathrm{d}\mu^i(y)=\mu^i(V_{x_k})$.
Finally, we have
\begin{align*}
    \min_{a} F(a,x) &= \frac{1}{N} \sum_{i=1}^N H(0_{\mathbb{R}^K}, (x_k)_{k=1}^K, (\mu^i(V_{x_k}))_{k=1}^K) \\
    &= \frac{1}{N} \sum_{i=1}^N\sum_{k=1}^K \int_{V_{x_k}} \|x_k-y\|^2\mathrm{d}\mu^i(y)\\
    &= \sum\limits_{k=1}^K \int_{V_{x_k}} \|x_k-y\|^2\mathrm{d}\bar{\mu}(y)\\
    &= W_2^2\left(\sum\limits_{k=1}^K \overline{\mu}(V_{x_k}) \delta_{x_k}, \overline{\mu}\right)
\end{align*}
which concludes the proof.
\end{proof}

\subsubsection{An Embedding with the OT Linearization}

 The Wasserstein space consisting of probability measures admitting a moment of order 2 endowed with the 2-Wasserstein distance  has  nice geometric properties such as existence of geodesics and a pseudo-Riemannian structure \cite{ambrosio2005gradient}. However, it is not a Hilbert space \cite{peyre2019computational}, and tools such as PCA or LDA (Linear Discriminant Analysis) for statistical learning cannot be directly applied on raw probability measures. The OT linearization \cite{wang2013linear} then consists in embedding probability measures in a linear space, allowing to apply standard statistical tools.  {\RL In particular it has been used to perform PCA on images modeled as probability measures \cite{basu2014detecting,wang2013linear}}. Given a set of probability measures $(\mu_i)_{1\leq i \leq N}$ and an absolutely continuous reference measure $\mu$, linearized OT consists in embedding the $\mu_i$'s in the tangent space $\mathcal{T}_{{\mu}}$ at $\mu$, defined as a subspace of the Hilbert space $L^2({\mu},\mathbb{R}^d) = \{ v: \mathbb{R}^d \rightarrow \mathbb{R}^d ~|~ \int_{\mathbb{R}^d} \|v\|^2 \mathrm{d}{\mu} < \infty \}$, endowed with the weighted $L^2$ inner product $\langle T_1, T_2 \rangle_{L^2({\mu})} = \int_{\mathbb{R}^d} T_1(x)T_2(x) \mathrm{d}{\mu}(x)$. A measure $\mu^i$ is therefore embedded into $L^2({\mu},\mathbb{R}^d)$ through the following logarithm map, by analogy with the Riemannian setting :
\begin{equation}\label{logarithm_map}
    \text{Log}_{\mu}: \mu^i \mapsto T_{\mu^i} - \text{Id},
\end{equation}
where $T_{\mu^i}$ is the Monge map such that  $T_{\mu^i\#}\mu = \mu_i$. Note that $\text{Log}_{\mu}(\mu)=0$, and that the logarithm preserves the Wasserstein distance to the reference measure $\mu$, that is $W_2(\mu,\mu^i) = \|\text{Log}_{\mu}(\mu^i)\|_{L^2(\mu)}$.

The linearized OT distance between $\mu^i$ and $\mu^j$ with reference measure $\mu$ is then defined as
\begin{equation}\label{linearizedW2}
d_{\text{LOT}}(\mu^i,\mu^j) = \|\text{Log}_{\mu}(\mu^i) - \text{Log}_{\mu}(\mu^j) \|_{L^2(\mu)}.
\end{equation}
 Following linearized OT, one can then apply standard statistical tools in a linear space to the  embedded data $\text{Log}_{\mu}(\mu^i)$.
In practice, we shall apply linearized OT to the set $(\nu^i)_{1 \leq i \leq N}$ of discrete measures \eqref{eq:measuresnu} obtained after the quantization step described in Section \ref{sec:quantization}. As a reference measure, we propose using the Wasserstein barycenter \eqref{wbary} of $(\nu^i)_{1 \leq i \leq N}$ denoted by $\nu_{\text{bar}}$. The measure $\nu_{\text{bar}}$ is thus discrete, and optimal Monge maps might not exist from $\nu_{\text{bar}}$ to $\nu^i$. In that case, one can approximate the map from an optimal transport plan  $P^i$ solution of \eqref{kanto} with the barycentric mapping  $T^i$ \cite{peyre2019computational} defined for $i=1,\ldots,N$ as

\begin{equation}\label{barymapping}
     T^i(x_k) = \frac{1}{\bar{a}_k} \sum_{\ell=1}^K x_\ell  P^i_{k\ell},
\end{equation}
where $\nu_{\text{bar}} = \sum_{k=1}^K \bar{a}_k \delta_{x_k}$ and $\nu^i = \sum_{k=1}^K a_k^i \delta_{x_k}$. Note that the barycentric mappings $(T^i)_{i=1}^N$ are well defined on the support of the reference measure $\nu_{\text{bar}}$. {\RL Applying linearized OT to a set of $N$ (typically in the DataML-Bordeaux study, we have $N\sim500$) FCMs would boil down to solving $N$ times the OT problem with sets of $n\sim10^5$ points. This motivates the need for the previously described mean measure quantization step, which allows to reduce the size of the point clouds, and consequently decrease the computational times.}

\subsection{Representation of Datasets with PCA}
\label{sec:PCA}

We use statistical inference methods that integrate the simultaneous analysis of a large number of FCM samples (for several patients, at the time of leukemia diagnosis and during post-treatment follow-ups).  Our aim is then to   discriminate FCM data during follow-up according to the level of significance of MRD using low-dimensional data representation in a linear space as described previously. Note that given the dimension and the number of observations from cytometry, solving the OT problem between raw FCM datasets is intractable. A key contribution of our work is then to summarize  multiple FCM using the quantization step in Section \ref{sec:quantization} based on the K-means of the  merged datasets.  This procedure will be shown to be faster than using FlowSOM  as proposed in \cite{vial2021unsupervised} for MRD detection.
In the following, we thus exclusively work with the measures $(\nu^i)_{i=1}^N$ defined in \eqref{eq:measuresnu}, where $K$ is much smaller than the original numbers of cells in the raw samples.

\subsubsection{Statistical Analysis on the K-means Weight Vectors}

As the measures $(\nu^i)_{i=1}^N$ share a common support, it is natural to compare them in terms of their weight vectors $ a^i = (\mu^i(V_{x_1}), \ldots, \mu^i(V_{x_K}))_{i=1}^N$, belonging to the probability simplex $\Sigma_K$. We can then apply the classical tools of statistical analysis of compositional data belonging to the probability simplex \cite{Aitchison83}.

\subsubsection{Statistical Analysis with Linearized OT}

 The new methodology that we propose to represent and cluster FCM datasets consists of:
\begin{enumerate} 
    \item Computing a Wasserstein barycenter $\nu_{\text{bar}}$ \ref{wbary} of the measures $(\nu^i)_{i=1}^N$, with a support of size $K$.
    \item Computing $N$ transport plans $(P^{i})_{i=1}^N$ \eqref{kanto} between $\nu_{\text{bar}}$ and $\nu^i$, and their respective barycentric projections $T_{\nu^i}$ as in \eqref{barymapping}.
    \item Embedding the measures $\nu^i$ in a finite vector space through the logarithmic map \eqref{logarithm_map}. We denote $v^i:=\text{Log}_{\nu_{\text{bar}}}(\nu^i)\in\mathbb{R}^{d\times K}$.
    \item Applying classical Euclidean PCA to the vectors $v^i$  to visualize the data variability.
\end{enumerate}

We thus map $\nu^i$ to a vector $v^i$ in $\mathbb{R}^{d \times K}$ that is endowed with an inner product weighted by $\nu_{\text{bar}}$.  By retaining the first two components of a PCA of the vectors $v^i$, we are able to visualize each FCM dataset as a point in a 2D plan. We claim that this is helpful to the biologists to visually analyze the proximity between datasets that translates to a certain phenotype proximity.

\subsection{Performance Evaluation}
\label{sec:perf}

For some cytometry datasets of the DATAML-Bordeaux study, ground truth labels  are available with the measures  (MRD-BioM), giving a percentage linked to quantities expressing the presence of blasts.  MRD-BioM is then defined as being positive when its value is higher than {\RL $0.02\%$} and negative otherwise. We will then refer to positive and negative follow-ups respectively. On the other hand, the FlowSOM method also provides a measure of MRD, MRD-FlowS, which is related to the percentage of blasts in the bone marrow. {\RL Recent  studies \cite{rodriguez2024optimal, zhangrisk} demonstrate  that the commonly used 0.1\% threshold in flow cytometry may not optimally differentiate  patients with low versus high relapse risk. These studies identify an  optimal prognostic threshold in the range of 0.01\% to undetectable levels. Based on these findings, we define MRD positivity as the  detection of at least 20 LAIP events within a 500,000-cell sample,  consistent with established statistical confidence thresholds for flow  cytometry.}

As a criterion to assess the performances of our approach and alternative methods, we use the silhouette score, which is a value for evaluating a clustering result. It measures how close an element is to its own cluster and how far it is from the other clusters. The silhouette scores ranges from -1 to 1, the highest value corresponding to a perfect clustering. The score of a point $x$ is defined as:
$$S(x) = \frac{b(x)-a(x)}{\max(b(x),a(x))}$$
where $a(x)$ is the average distance between $x$ and the other points of its clusters, and $b(x)$ is the average distance between $x$ and points of other clusters. The silhouette score of a clustering is then the average of the silhouette scores of all the points in the dataset.

\subsection{Alternative Method}
\label{sec:KME}

We benchmark our approach with the popular kernel mean embedding (KME)  technique \cite{muandet2017kernel}  for statistical learning from multiple high-dimensional probability distributions by representing them  in a reproducing kernel Hilbert space (RKHS) $\mathcal{H}$ with inner product $\langle \cdot,\cdot\rangle_{\mathcal{H}}$. Given a kernel $k(x,y) = \langle \phi(x),\phi(y)\rangle_{\mathcal{H}}$ (a positive definite and symmetric function) associated to a feature map $\phi: \mathcal{X} \rightarrow \mathcal{H}$,  KME is defined as the following mapping
\begin{align}
   \Phi: \mathcal{M}_+^1(\mathcal{X})&\rightarrow \mathcal{H} \nonumber\\
   \mu &\rightarrow \int_{\mathcal{X}} k(x,\cdot)\mathrm{d}\mu(x).\label{kme}  
\end{align}
that embeds the set $\mathcal{M}_+^1(\mathcal{X})$ of probability measures supported on $\mathcal{X}$ into $\mathcal{H}$.
Under a characteristic  assumption on the kernel (namely that that $\|\Phi(\mu)-\Phi(\nu) \|_{\mathcal{H}} = 0$ if and only if $\mu=\nu$) KME defines a distance  $\|\Phi(\mu)-\Phi(\nu) \|_{\mathcal{H}}$, called a Maximum Mean Discrepancy (MMD), on the space of probability distributions.

To implement KME for high-dimensional probability measures, we rely on random Fourier features \cite{rahimi2007random} in order to approximate the embedding \eqref{kme}. If $k$ is a translation invariant kernel, then there exists a function $\psi$ such that $k(x,y) = \psi(x-y)$, and Bochner's theorem \cite{bochner1939additive} implies that
\begin{equation}\label{bochner}
    \psi(x-y) = \int_{\mathbb{R}^d} e^{i\omega^T(x-y)}\mathrm{d}\Lambda_k(\omega)
\end{equation}
where $\Lambda_k$ is a Borel measure depending on the kernel $k$. Random Fourier features provide an approximation of $k(x,y)$  with a feature map $\tilde{\varphi}(x)\in\mathbb{R}^s$ such that $\tilde{\varphi}(x)^T\tilde{\varphi}(y)\approx k(x,y)$, which is defined as :

$$\tilde{\varphi}(x) = (\sin(\omega_1^Tx),\cdots,\sin(\omega_{s/2}^Tx),\cos(\omega_1^Tx),\cdots,\cos(\omega_{s/2}^Tx))$$ and where the $\omega_i$'s are independently sampled from $\Lambda_k$.  In this particular framework, KME becomes the following mapping between probability measures and a  linear space of finite dimension $s$:
\begin{align}
   \tilde{\Phi}: \mathcal{M}_+^1(\mathcal{X})&\rightarrow \mathbb{R}^s \nonumber\\
   \mu &\rightarrow \int_{\mathcal{X}} \tilde{\varphi}(x)\mathrm{d}\mu(x).\label{kme_fourier}  
\end{align}
Note that when $\mu$ is a discrete measure, the above integral is a sum over the support points of $\mu$.
In our implementations, we use a radial basis kernel with $\Lambda_k=\mathcal{N}(0,\sigma^{-2}\mathrm{Id})$.

\section{Results}
\label{sec:results}

\subsection{HIPC Dataset}

\subsubsection{PCA  for Visualization}

In Figure  \ref{PCA_viz}, we first compare two-dimensional representations of the HIPC dataset obtained by a PCA  computed from four different embeddings:
\begin{enumerate}
    \item Kernel Mean Embedding (\emph{KME})
    \item Log-ratio transform of the compositional data  $(a^i)_{i=1}^N$ (\emph{K-Means + comp})
    \item OT linearization of the K-means clustered measures $(\nu^i)_{i=1}^N$ (\emph{K-Means + LinW2})
    \item OT linearization of the FlowSOM clustering of $(\mu^i)_{i=1}^N$ (\emph{FlowSOM + LinW2}).
\end{enumerate}

{\RL 
 {\CT When selecting the number of clusters $K$, one aims at compromising between performance and computational efficiency.} In Figure \ref{fig:choice_parameters}, we compare the silhouette score (representing performance) and the execution time of our method on the HIPC dataset when $K$ varies from 10 to 250. The silhouette score is computed in the three-dimensional space obtained after a three-component PCA applied on the LOT embedded data. We chose to keep three components with an elbow criterion. We observe that for $K\geq50$, the silhouette score {\CT stabilizes} around 0.66 whereas the time of execution grows rapidly. Moreover, the authors {\CT of the FlowSOM method \cite{vial2021unsupervised}} chose $K=64$ for the number of clusters, which in our analysis corresponds to a good trade-off between performance and computational time. In the same Figure \ref{fig:choice_parameters}, we also compare different choices for the reference measure {\CT in the LOT embedding \eqref{logarithm_map}. The first natural choice is the Wasserstein barycenter $\nu_{\mathrm{bar}}$ \eqref{wbary} of the $\nu^{i}$'s. We also consider as a reference measure the uniform distribution supported on the 
the K-means centers $\nu_{\mathrm{unif}} = (1/K)\sum_{k=1}^K\delta_{x_k}$ and a measure $\nu_{\mathrm{rand}} = \nu^{i}$ randomly chosen among the $\nu^{i}$'s.} Another possibility introduced in \cite{khurana2023supervised} is to choose multiple references for multiple embeddings. We also tested this option by randomly choosing two indexes $i,j$ and performing LinW2 with respect to {\CT the product space} $L^2(\nu^{i})\times L^2(\nu^{j})$. We {\CT observe} that choosing the Wasserstein barycenter as the reference measure performs better in general. This could be computationally expensive but {\CT remains} reasonable when $K=64$.}
 
 \begin{figure}[htbp]
\centering
\includegraphics[width=0.5\textheight]{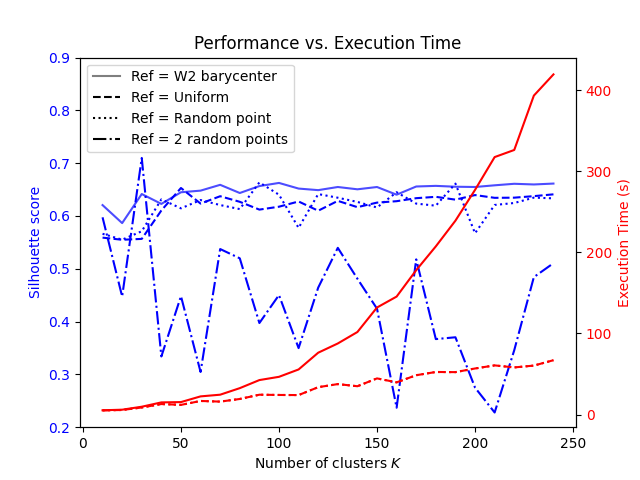}
\caption{{\RL HIPC dataset : Silhouette score and execution time of our method as a function of the number of clusters. The red dotted line represents the execution time for the reference measure chosen as either uniform, random among the $\nu^{i}$'s or the concatenation of two measures $(\nu^{i},\nu^{j})$.}}\label{fig:choice_parameters}
\end{figure}
 
 For the KME, two parameters are required: the bandwidth $\sigma$ for RBF (Radial Basis Function) kernel and the dimension $s$ of the Fourier features. In order to fairly compare results, we choose $s=K\times d = 448$ as the linearization sends measures to vectors in $ \mathbb{R}^{64\times7}=\mathbb{R}^{448}$. Regarding the bandwidth of the kernel, we set $\sigma=5$, which allows to best discriminate the patients using PCA. 

In Figure \ref{PCA_viz}, we use K = 64. The OT linearization allows to obtain a visually accurate discrimination of the datasets according to the patients' labels whereas the representations obtained by KME and the compositional data analysis struggle to clearly separate datasets. Within the patients' clusters, we can even distinguish clusters according to the three stained replicates. The OT linearization on data transformed by K-means or FlowSOM shows similar visual results.

\begin{figure}[htbp]
\includegraphics[width=0.7\textheight]{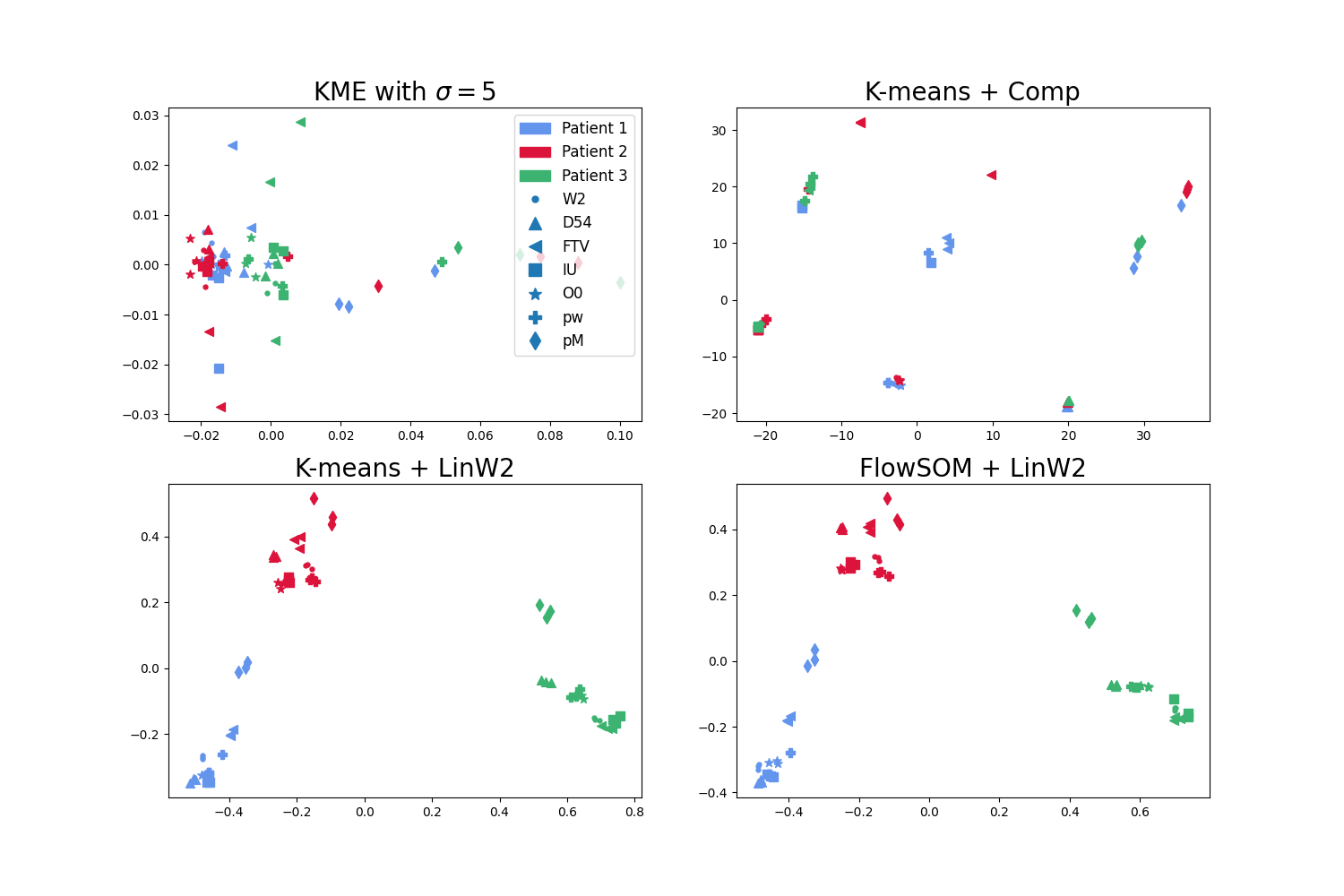}
\caption{Two-dimensional representations of the HIPC data using PCA on different embeddings  : (top left) KME, (top right) K-Means + comp, (bottom left) K-Means + LinW2, (bottom right) FlowSOM + LinW2. Each color corresponds to a patient. Each marker encodes the laboratory where the data was processed.}\label{PCA_viz}
\end{figure}

\subsubsection{Silhouette Scores}

In Table \ref{silhouette}, we report silhouette scores to evaluate how well each of the four embedding discriminates between clusters. We note that silhouette scores should not be compared for different dimensions: as distances become similar in higher dimensions, the score naturally worsen in those cases. Kernel mean embedding performs poorly and K-means followed by compositional analysis gives bad results when $K>16$. As expected from the PCA visualizations, the silhouette scores for OT linearization on FlowSOM and K-means are highly similar.

{\RL
\begin{table*}[htbp]
\caption{Silhouette scores for the different embeddings on HIPC. For each method, we compute the score either directly on raw data or on the data projected onto the first {\RL three (chosen with the elbow criterion) } principal components of the PCA. The best score per row is indicated in bold.}\label{silhouette}
\tabcolsep=0pt
\begin{tabular*}{\textwidth}{@{\extracolsep{\fill}}lcccccccccc@{\extracolsep{\fill}}}
\hline%
& \multicolumn{2}{@{}c@{}}{KME} & 
\multicolumn{2}{@{}c@{}}{K-Means + Comp}& 
\multicolumn{2}{@{}c@{}}{FlowSOM + LinW2}& 
\multicolumn{2}{@{}c@{}}{K-Means + LinW2} \\
\cline{2-3}\cline{4-5}\cline{6-7}\cline{8-9}%
$K$ & Raw & PCA & Raw & PCA & Raw & PCA & Raw & PCA\\
None  & -0.02  & -0.1 &  &  &  & & & \\
16  &  &  & 0.47 & 0.48 & 0.46 & 0.55 & 0.46 & \textbf{0.57}\\
32  &  &  & 0.01 & 0.51 & 0.34 & \textbf{0.62} & 0.38 & \textbf{0.62} \\
64  &  &  & 0.05 & 0.04 & 0.29 & 0.64 & 0.28 & \textbf{0.65} \\
128  &  &  & 0.06 & 0.04 & 0.24 & \textbf{0.66} & 0.24 & 0.64 \\
256  &  &  & 0.05 & 0.1 & 0.21 & \textbf{0.65} & 0.2 & \textbf{0.65}\\
\hline
\end{tabular*}
\end{table*}
}

\subsubsection{Execution Time}

Computations were performed using HPC resources from the MCIA (M\'esocentre de Calcul Intensif Aquitain) of the Universit\'e de Bordeaux. In Figure  \ref{time_comparison}, we compare the computational times of each embedding. K-means takes the upper hand, being two  to three times faster than FlowSOM. KME is however significantly slower. The linearization of optimal transport takes less than 2 seconds no matter the dimension. The most time-consuming step in our process is computing the reference measure $\nu_{\text{ref}}$, which we have defined as the barycenter of all $N$ measures $(\nu^i)_{i=1}^N$. Performing a Wasserstein barycenter is indeed costly, especially as the number K of clusters increases. However, for our dimensionality reduction purpose, K is kept small. In any case, one could set the reference measure as the uniform measure supported on the K-means centers $\nu_{\text{ref}}= \frac{1}{K} \sum_{k=1}^K \delta_{x_k}$. This procedure is significantly faster, but leads to a slight drop in the silhouette score.

\begin{figure}[htbp]
\centering
\includegraphics[width=0.7\textheight]{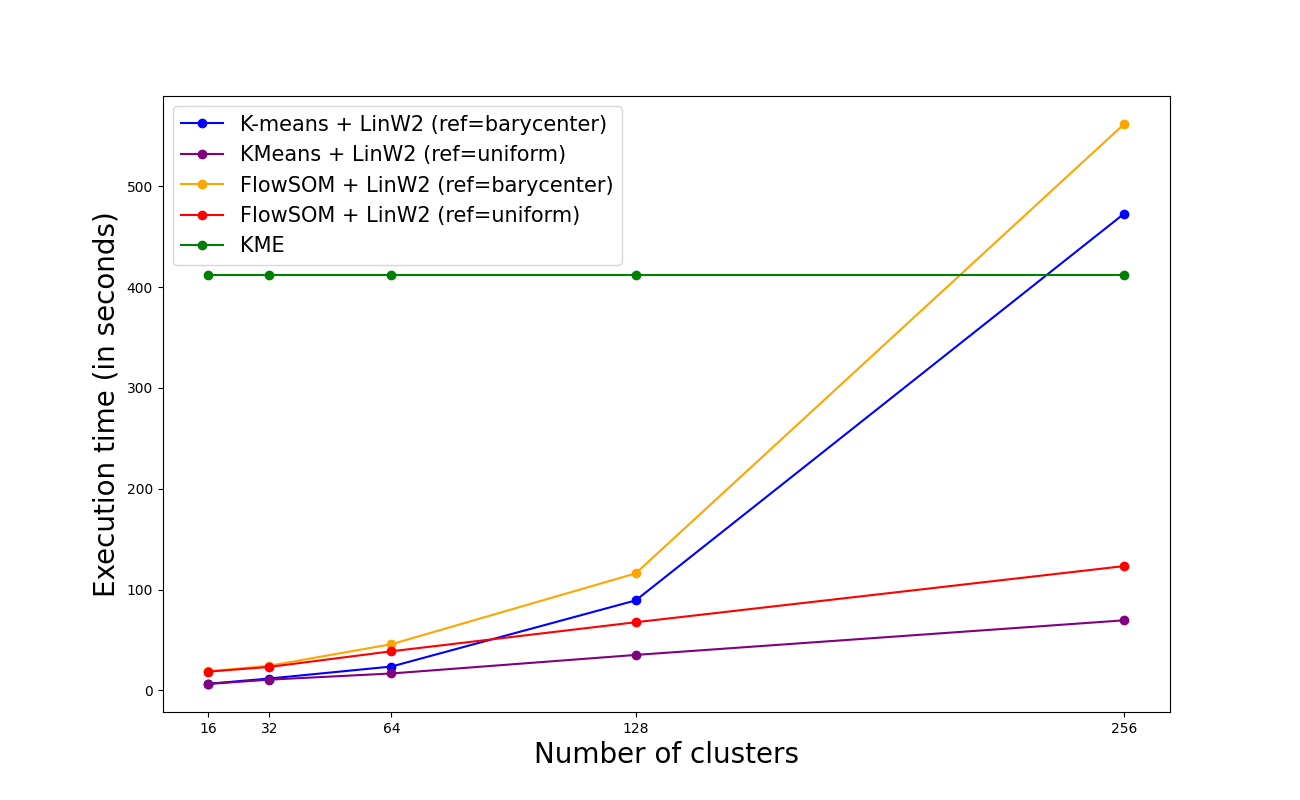}
\caption{Comparison of time of execution for the different embeddings.}\label{time_comparison}
\end{figure}

\subsection{DATAML-Bordeaux Dataset}

\subsubsection{Representation of a Patient's Data}

The first step of our analysis is to perform a K-means on the merged DATAML-Bordeaux dataset. This results in $N$ measures $\nu^i = \sum_{k=1}^K a^i_k \delta_{x_k}$ with a common support $\{x_1,\cdots,x_K\}\subset\mathbb{R}^d$. To compare ourselves with FlowSOM \cite{van2015flowsom}, we propose to visualize the output of K-means with a tree.  We construct a graph whose nodes are the K-means centers $x_1,\ldots,x_K$. Each pair of nodes $x_i$ and $x_j$ is connected by an edge with its weight corresponding to the distance $\|x_i-x_j\|^2$. We then compute a minimum spanning tree (MST) from this graph. One can visualize this MST corresponding to different data of one patient in Figure \ref{mst}. The size of the nodes vary from one MST to another, as it represents the number of cells in the corresponding clusters. Comparing a MST with the common MST reveals biological specificities. In the case of Figure \ref{mst}, we compare the diagnosis of the patient, two of their follow-ups, and a NBM. It is interesting to notice that the node at the center of the tree holds quite some mass in the diagnosis and almost none in the NBM. One could speculate that this node is a cluster of blasts. This assumption is also supported by the fact that this node has more mass in the second follow-up where the MRD-BioM value is higher, than in the first follow-up where the MRD-BioM value is lower.

\begin{figure}[htbp]
\centering
\includegraphics[width=0.7\textheight]{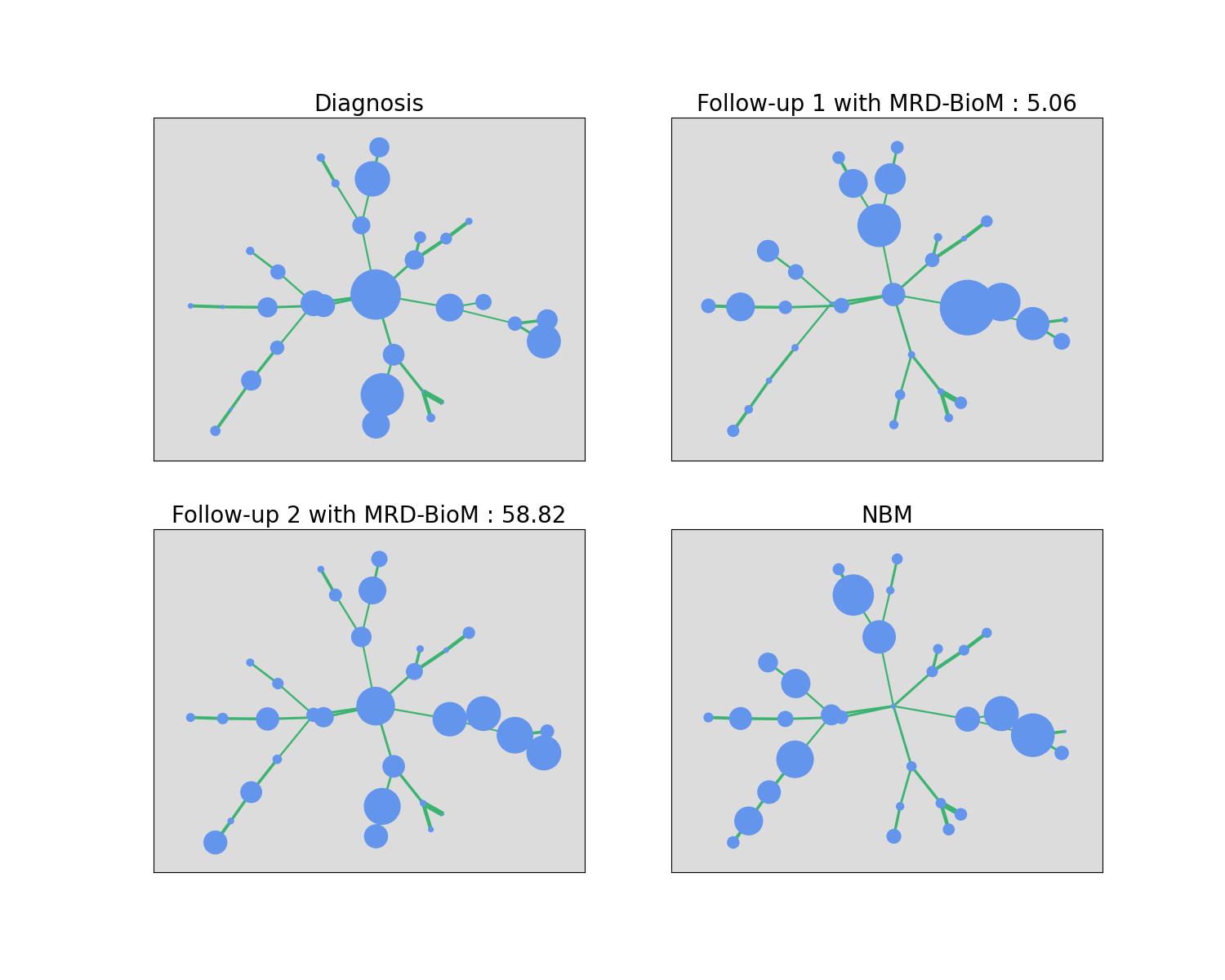}
\caption{Minimum Spanning Tree for the diagnosis and two follow-up measurements of one patient of the DATAML-Bordeaux dataset, and one normal bone marrow.}\label{mst}
\end{figure}

\subsubsection{PCA}

In Figure \ref{PCA_DataML}, we compare two-dimensional representations of the dataset obtained by a PCA computed from the KME, the log-ratio transform of compositional data of K-means and the linearization of OT on K-means clustering. The OT linearization allows us to visually discriminate between diagnosis, normal bone marrow and follow-ups. It is however harder to distinguish positive from negative follow-ups.

In Figure \ref{PCA_MRD_size}, we visualize the amount of MRD in the follow-ups for the K-Means + LinW2 PCA method, where the size of the follow-up points represents the MRD-BioM value for the patient at the time that the sample was collected. We observe that high MRD-BioM follow-ups are pulled towards the set of diagnoses. Regarding the negative follow-ups, these are evidently more directed towards NBM and clearly separated from the set of diagnoses. We also visualize the MRD-FlowS predictions with the FlowSOM method in Figure \ref{PCA_fsom_pred}, where we observe the same effect: high MRD-FlowS points lie in the diagnosis set whereas lower MRD-FlowS are closer to normal bone marrow points. Our  low-dimensional representation of multi-patient FCM datasets is therefore consistent with the values of MRD-FlowS.

\begin{figure}[htbp]
\centering
\includegraphics[width=0.6\textheight]{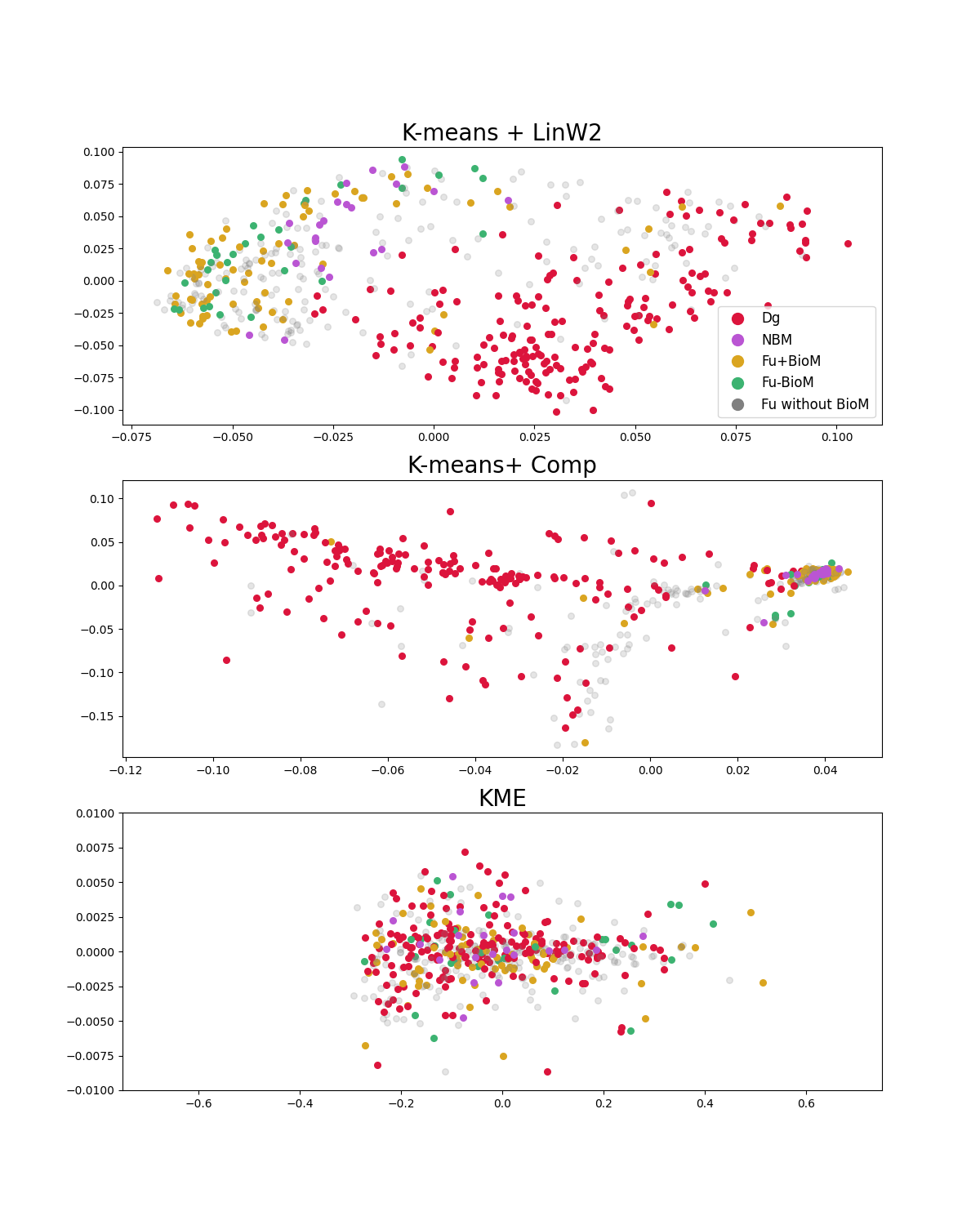}
\caption{Two-dimensional representations of the DATAML-Bordeaux datasets using PCA on different embeddings : (top) K-Means + LinW2, (middle) K-Means + comp, (bottom) KME. Each dot represents a FCM sample. The colors encore the nature of the data : (red) diagnostic, (purple) normal bone marrow, (yellow) positive follow-up, (green) negative follow-up, (gray) no information on the MRD-BioM.} \label{PCA_DataML}
\end{figure}

\begin{figure}[htbp]
\centering
\includegraphics[width=0.7\textheight]{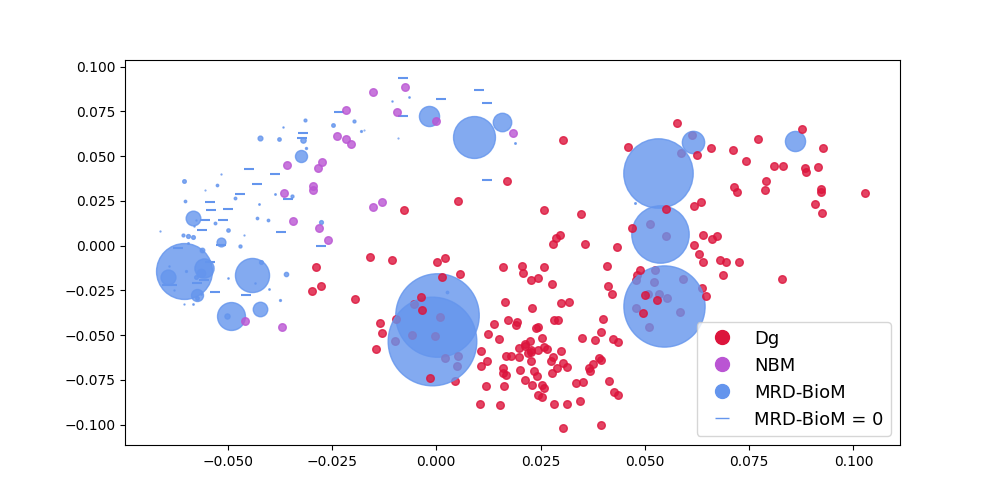}
\caption{Two-dimensional representation of the DATAML-Bordeaux datasets using PCA on the K-Means + LinW2 embedding. The size of the follow-ups is scaled with a logicle transformation of the MRD-BioM values ($\log(y)$ when $y\geq 1$, $y$ else).\label{PCA_MRD_size}}
\end{figure}

\begin{figure}[htbp]
\centering
\includegraphics[width=0.7\textheight]{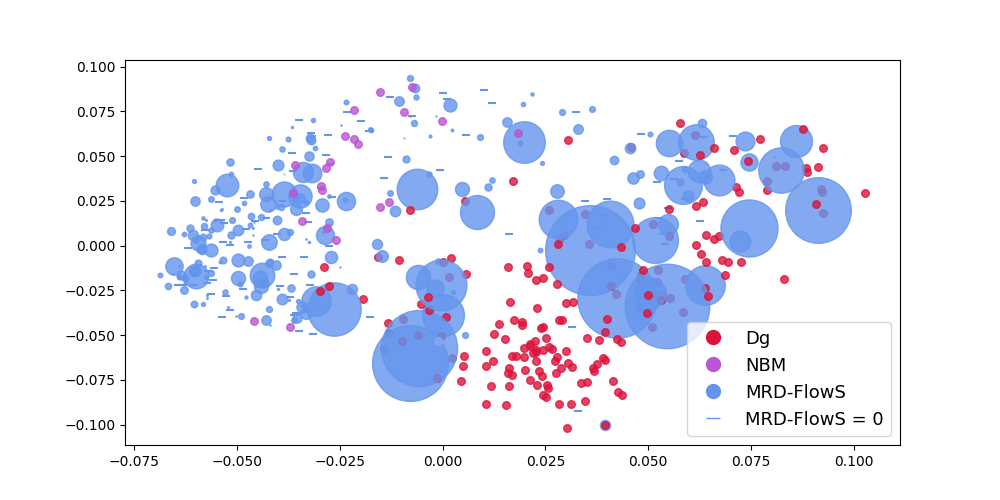}
\caption{Two-dimensional representation of the DATAML-Bordeaux datasets using PCA on the K-Means + LinW2 embedding. The size of the follow-ups is scaled to the values of MRD-FlowS.}\label{PCA_fsom_pred}
\end{figure}

\subsection{Silhouette Scores}

We finally report in Table \ref{silhouette_dataML}  silhouette scores to evaluate how the various embeddings of the DATAML-Bordeaux dataset displayed in Figure \ref{PCA_DataML}  allow us to discriminate clusters. We consider the following clusters: diagnosis, normal bone marrow, positive MRD-BioM follow-ups, and negative MRD-BioM follow-ups. We leave out the follow-ups for which we do not have a value for MRD-BioM. From the results reported in Table \ref{silhouette_dataML}, it can be seen that the clustering obtained with  OT linearization is still the best one.
 
\begin{table*}[htbp]
\caption{Silhouette scores for the different embeddings on DATAML-Bordeaux. For each method, we compute the score either directly on raw data or on the data projected onto the first {\RL three (chosen by the elbow criterion)} principal components of the PCA. The best score per row is indicated in bold.}\label{silhouette_dataML}

\tabcolsep=0pt
\begin{tabular*}{\textwidth}{@{\extracolsep{\fill}}lcccccccccc@{\extracolsep{\fill}}}
\hline%
& \multicolumn{2}{@{}c@{}}{KME} & 
\multicolumn{2}{@{}c@{}}{K-Means + Comp}& 
\multicolumn{2}{@{}c@{}}{K-Means + LinW2} \\
\cline{2-3}\cline{4-5}\cline{6-7}%
$K$ & Raw & PCA & Raw & PCA & Raw & PCA\\
None  & -0.19  & -0.22 &  &  &  & \\
16  &  &  & -0.015 & -0.2 & \textbf{0.04} & \textbf{0.04} \\
32  &  &  & -0.09 & -0.04 & 0.07 & \textbf{0.09} \\
64  &  &  & -0.004 & -0.1 & 0.06 & \textbf{0.1}  \\
128  &  &  & -0.03 & 0 & 0.06 & \textbf{0.08}  \\
256  &  &  & -0.02 & 0.02 & 0.04 & \textbf{0.08} \\
\hline
\end{tabular*}
\end{table*}

{\RL
\subsection{Supervised Classification}

One can leverage our dimension reduction method for classification purposes. After applying a three-component PCA on the different embeddings, we consider two classes : positive follow-ups (Fu+) and negative follow-ups (Fu-). For each embedding, we apply a three-component PCA on the embedded data, train a logistic regression model and predict either a positive or negative follow-up with a leave-one-out cross validation technique. We obtain the confusion matrices shown in Figure \ref{fig:conf_matrices}. In order to compare the different methods, we also provide in Figure \ref{fig:conf_matrices} the balanced accuracy scores, which allow evaluating classifier performances when dealing with imbalanced data, that is when one of the classes is more present than the other. We observe that the LOT method outperforms the others in terms of this score.

\begin{figure}[htbp]
\centering
\includegraphics[width=0.75\textheight]{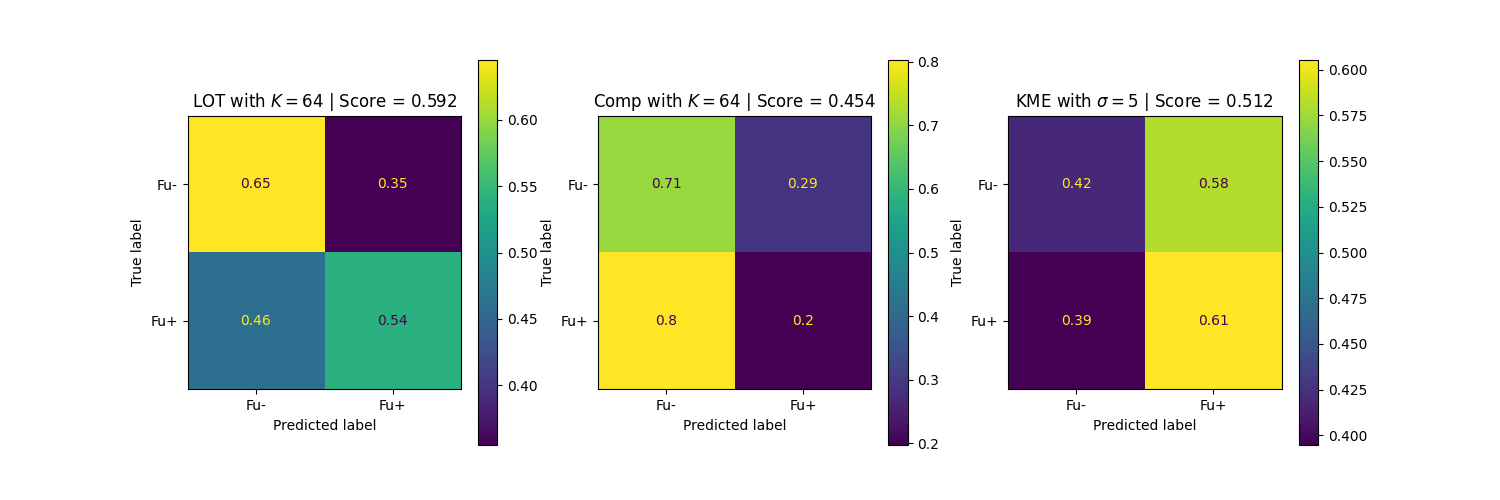}
\caption{{\RL Confusion matrices and balanced accuracy scores for the logistic regression classifier on the (left) K-means+LinW2+PCA, (middle) K-means+Comp+PCA and (right) KME+PCA models. The balanced accuracy score is defined as the arithmetic mean between specificity and sensitivity : $\frac{1}{2}\bigl(\frac{TP}{TP+FN}+\frac{TN}{TN+FP}\bigr).$}}\label{fig:conf_matrices}
\end{figure}


}

\section{Discussion} \label{sec:dis}

In this study, we have introduced a novel approach,  via mean measure quantization and linearized OT, for low-dimensional embedding in a linear space of multi-patient FCM datasets viewed as discrete probability measures with large dimensional supports. We assessed the benefits of OT linearization over compositional data analysis and kernel mean embedding by showing that our OT-based approach yields the best silhouette scores for patient clustering. The biological relevance of our OT-based embedding  has also been assessed by relating the consistency of our low-dimensional representation of multi-patient FCM datasets to the level of estimated MRD in follow-up using either  molecular biology or  the FlowSOW method. Interestingly, patients with positive follow-ups and large MRD values tend to be clustered, in our low-dimensional representation, with data points representing FCM from diagnosis. These results suggest that using OT for  statistical learning from multi-patient FCM datasets is an interesting research direction in order to help the experts with MRD assessment. Combined with the FlowSOM method, our approach could enhance the accuracy of relapse prognosis which is critical for making informed treatment decisions for AML  patients. {\RL Indeed, if biologists have predicted a positive MRD with their method and the corresponding sample is close to the diagnosis in the two-dimensional embedding space, they can confirm their prediction. On the other hand, if they have predicted negative MRD and the sample is among diagnosis, they might want to take another look at their analysis.}

\section{Acknowledgments}

This work benefited from financial support from the French government managed by the National Agency for Research under the France 2030 program, with the reference ANR-23-PEIA-0004. We would like to thank the data management unit of Toulouse University Hospital and the CAPTOR (Cancer Pharmacology of Toulouse Oncopole and Region) project (ANR-11-PHUC-001) for its financial support enabling e-CRF.

\bigskip
\noindent
\textbf{Conflict of interest.} The authors have no conflict of interest to declare.

\bigskip
\noindent
\textbf{Author contributions.} Erell Gachon: conceptualization, methodology, writing - original draft, investigation, writing - review and editing. J\'er\'emie Bigot: conceptualization, methodology, writing - original draft, investigation, writing - review and editing, supervision, funding acquisition. Elsa Cazelles: conceptualization, methodology, writing - original draft, investigation, writing - review and editing, supervision, funding acquisition. Audrey Bidet : resources, data curation, validation, writing - review and editing. Jean-Philippe Vial: resources, data curation, validation, writing - review and editing. Pierre-Yves Dumas: resources, data curation, validation, writing - review and editing. Aguirre Mimoun: resources, data curation, validation, writing - review and editing.

\bibliographystyle{plain}
\bibliography{reference}

\end{document}